\documentclass[11pt]{article}

\usepackage[margin=1in]{geometry}
\usepackage[T1]{fontenc}
\usepackage[utf8]{inputenc}
\usepackage{lmodern}

\usepackage{microtype}
\usepackage{amsmath,amssymb,amsthm}
\usepackage{graphicx}
\usepackage{xcolor}
\usepackage{booktabs}
\usepackage{url}
\usepackage[hidelinks]{hyperref}

\newtheorem{theorem}{Theorem}
\newtheorem{corollary}[theorem]{Corollary}
\newtheorem{remark}{Remark}

\title{In-Context Learning Under Regime Change}

\author{
Carson Dudley$^{1,*}$ \and
Yutong Bi$^{2,*}$ \and
Xiaofeng Liu$^{3}$ \and
Samet Oymak$^{2}$
}

\date{}

\begin{document}
\maketitle

\begin{center}
{\small
$^1$ Department of Mathematics, University of Michigan\\
$^2$ Department of Electrical Engineering and Computer Science, University of Michigan\\
$^3$ Michigan Institute for AI and Data in Society, University of Michigan\\[4pt]
$^*$ Equal contribution\\[4pt]
{\small \texttt{\{cdud, yutongb, xiaofliu, oymak\}@umich.edu}}}
\end{center}

\maketitle

\begin{abstract}

Non-stationary sequences arise naturally in control, forecasting, and decision-making. The data-generating process shifts at unknown times, and models must detect the change, discard or downweight obsolete evidence, and adapt to new dynamics on the fly. Transformer-based foundation models increasingly rely on in-context learning for time series forecasting, tabular prediction, and continuous control. As these models are deployed in non-stationary environments, understanding their ability to detect and adapt to regime shifts is important. We formalize this as an in-context change-point detection problem and formally establish the existence of transformer models that solve this problem. Our construction demonstrates that model complexity, in layers and parameters, depends on the level of information available about the change-point location, from no knowledge to knowing exact timing. We validate our results with experiments on synthetic linear regression and linear dynamical systems, where trained transformers match the performance of optimal baselines across information levels. We also show that encoding and incorporating changepoint knowledge indeed improves the real-world performance of a pretrained foundation models on infectious disease forecasting and on financial volatility forecasting around Federal Open Market Committee (FOMC) announcements without retraining, demonstrating practical applicability to real-world regime changes.

\end{abstract}

\section{Introduction}


Non-stationary environments are the norm rather than the exception in real-world sequential estimation. In control systems, component failures or environmental changes cause abrupt shifts in plant dynamics \cite{willisky1976, basseville1993detection}. In the Markov jump linear system (MJLS) framework \cite{costa2005discrete}, such regime switches are modeled as transitions of a discrete mode variable, and estimation quality depends critically on what is known about the switching signal. In financial markets, earnings reports and federal policy announcements can abruptly shift volatility regimes \cite{hamilton1989new, ang2011regime}. In epidemiology, policy interventions alter disease transmission dynamics mid-outbreak \cite{dudley2025sparse}. A common theme is a change-point detection and adaptation problem: the data-generating process shifts at an unknown time, and any estimator must identify the change, discard obsolete evidence, and adapt to new dynamics on the fly.

Classical approaches to change-point detection are well-studied. CUSUM procedures \cite{page1954} and their extensions \cite{lai1995} provide sequential tests with known optimality properties. Bayesian online change-point detection \cite{adams2007} maintains a posterior over run lengths, enabling principled adaptation to regime shifts. These methods, however, typically require explicit specification of the model class and change-point prior, and are designed for a single detection task.

A parallel development in machine learning has demonstrated that transformers can solve estimation problems in-context, adapting to new tasks purely from prompt examples without any parameter updates \cite{vaswani2023attentionneed, brown2020, garg2022}. Theoretical work has shown that transformer forward passes can implement gradient descent \cite{vonoswald2023}, least-squares and ridge regression \cite{akyurek2023, li2023}, and Bayesian model averaging \cite{zhang2023icl_bma}, with recent analyses of training dynamics toward such solutions \cite{zhang2023trained}. Transformer-based foundation models now underpin deployed systems for time series forecasting \cite{timesfm, chronos}, tabular prediction \cite{tabpfn}, and other sequential tasks where non-stationarity is pervasive. Yet nearly all in-context learning theory assumes stationary prompts drawn from a single task. The most closely related work is \cite{gatingisweighting}, which studies heterogeneous prompts using gated linear attention but without temporal sequential regime change.

This stationarity assumption is increasingly at odds with practice. As foundation models are deployed in settings where regime shifts are common, understanding their ability to detect and adapt to abrupt changes becomes important---not only from a machine learning perspective, but from the classical estimation perspective of what information structure the estimator requires.

\textbf{Contributions.} We formalize in-context change-point detection as a framework for studying transformers' ability to adapt to non-stationary sequences. Our contributions are:

\begin{enumerate}
    \item \textbf{Problem formulation.} We define a family of piecewise-linear in-context learning problems in which the data-generating process switches at an unknown change point. We consider a hierarchy of information levels---from no knowledge of the change-point location to exact timing---and study how positional encoding can communicate this side information to the model.

    \item \textbf{Constructive theory.} We provide explicit transformer constructions that solve the in-context change-point detection and adaptation problem. Our constructions demonstrate that the required model complexity (depth and width) depends on the level of side information available about the change point, establishing a capability-complexity tradeoff for non-stationary in-context learning.

    \item \textbf{Synthetic validation.} We train GPT-2 style transformers on piecewise-linear regression and piecewise-linear dynamical system tasks with stochastic change points. Across both settings, trained transformers match the performance of theoretically optimal baselines, confirming that the models learn to perform implicit change-point detection and adaptation in-context.

    \item \textbf{Real-world validation.} We demonstrate that our positional encoding methods improve the performance of a pretrained time series foundation model on infectious disease forecasting and financial volatility forecasting around Federal Open Market Committee (FOMC) announcements, without any retraining. This confirms that our methods offer practical value for deployed foundation models.
\end{enumerate}

\textbf{Related work.}
Our work lies at the intersection of change-point detection and in-context learning theory. Classical methods such as Bayesian online change-point detection \cite{adams2007} provide optimal procedures under explicit models, but do not address adaptation within learned models or in-context computation. On the in-context learning side, prior work has shown that transformers implement implicit learning algorithms, including gradient-based updates and least-squares regression \cite{garg2022,vonoswald2023,akyurek2023}, with recent results establishing approximation and training dynamics toward such solutions \cite{li2023,zhang2023trained}. However, these works assume stationary prompts drawn from a single task. The most closely related work is \cite{gatingisweighting}, which studies heterogeneous prompts, but without temporal structure or sequential regime change. In contrast, we study non-stationary sequences with an unknown change point and give constructive transformer architectures for in-context detection and adaptation, with complexity depending on the level of change-point information.

\begin{figure*}[htbp]
    \centering
    \includegraphics[width=\textwidth]{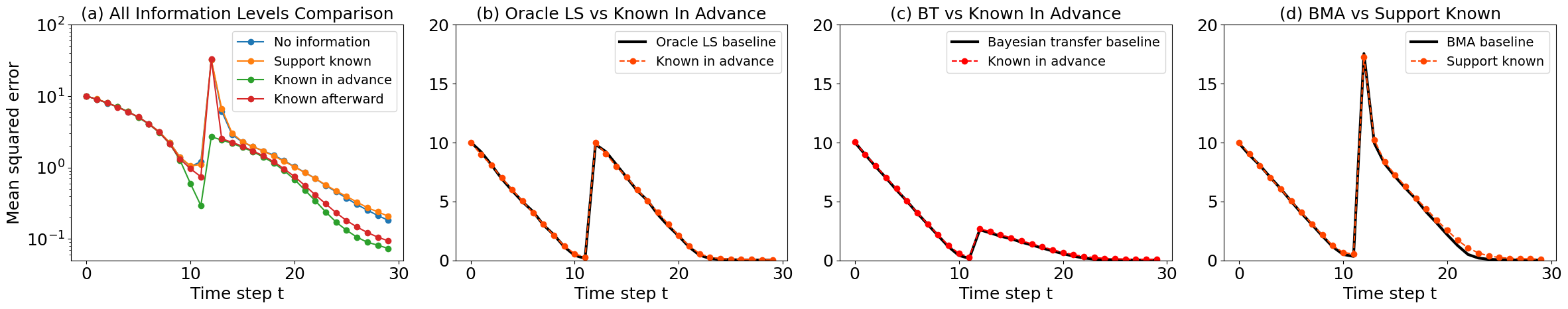}
    \caption{Piecewise-linear regression results. Mean squared error averaged over 5,000 test trajectories as a function of time step $t$, with the change point at $n_1^* = 12$ (vertical dashed line). \textbf{(a)}~Comparison of all four information levels on a log scale. Before the change point, the known-in-advance variant achieves the lowest error by isolating pre-change data. All variants spike at the transition; the no-information and support-known variants recover more slowly than the informed variants. \textbf{(b)}~The known-in-advance transformer closely matches the oracle least-squares baseline, which knows the true change point and fits only on data from the active regime. \textbf{(c)} Under a transfer task where $w_2 = -w_1 + \varepsilon$, the known-in-advance transformer matches a transfer ridge regression baseline that leverages pre-change data through the prior, demonstrating that the model learns to exploit the relationship between regimes rather than simply discarding old data. \textbf{(d)}~The support-known transformer closely tracks the Bayesian model averaging baseline, confirming that transformers with partial change-point information learn to perform approximate posterior averaging over candidate change-point locations in-context.}
    \label{fig:linreg}
\end{figure*}

\section{Transformers Can Implement Bayesian Change-Point Adaptation}
\label{sec:theory}
We show that a causal transformer can approximate the Bayesian model-averaged (BMA) predictor for the piecewise-linear change-point problem defined in~\eqref{eq:model}, and characterize how model complexity depends on available side information. We present the construction and proof sketches here; complete proofs appear in the extended version \cite{arxiv_version}.

\subsection{Problem Setup}

We observe $N$ input-label pairs $\{(x_i, y_i)\}_{i=1}^N$ generated by two linear regression tasks separated by an unknown change point $n_1^*$:
\begin{equation}
y_i = \begin{cases} \langle w_1, x_i \rangle + \varepsilon_i, & i \leq n_1^*, \\ \langle w_2, x_i \rangle + \varepsilon_i, & i > n_1^*. \end{cases}
\label{eq:model}
\end{equation}
We assume bounded inputs: $\|x_i\| \leq B_x$ and $|y_i| \leq B_y$ for known constants. The noise is Gaussian, $\varepsilon_i \sim \mathcal{N}(0, \sigma^2)$, and we place a Gaussian prior $w_1, w_2 \sim \mathcal{N}(0, \lambda^{-1} I_d)$ on the regression weights. The change point satisfies $n_1^* \in \mathcal{K}_t \subseteq \{1, \dots, t{-}1\}$, where $\mathcal{K}_t$ is the set of admissible change-point locations at time $t$, with $|\mathcal{K}_t| \leq R$. We assign a uniform prior $\pi_k = 1/|\mathcal{K}_t|$ over candidates.

\subsection{Main Result}

BMA is the optimal prediction strategy under squared error loss when the true model lies within the hypothesis class \cite{min1993bayesian, hoeting1999bma}. Recent work has shown that transformers implicitly implement BMA for stationary in-context learning \cite{zhang2023icl_bma}; our construction extends this to the non-stationary setting with an unknown change point. The BMA predictor maintains a posterior over change-point hypotheses and combines their predictions:
\begin{equation}
\hat{y}_t^{\mathrm{BMA}} = \sum_{k \in \mathcal{K}_t} \alpha_k(t) \, m_k(t),
\label{eq:bma}
\end{equation}
where $m_k(t) = x_t^\top \mu_k$ is the posterior predictive mean under hypothesis $k$ and $\alpha_k(t) \propto \pi_k \, p(Z_{<t} \mid k)$ is the posterior change-point weight. Under the Gaussian prior and noise model, both quantities depend on the sufficient statistics $A_k(t) = \sum_{i=k+1}^{t-1} x_i x_i^\top$, $b_k(t) = \sum_{i=k+1}^{t-1} x_i y_i$, and $c_k(t) = \sum_{i=k+1}^{t-1} y_i^2$ for the post-change segment under hypothesis~$k$.

Our construction shows that a causal transformer can compute these quantities and approximate~\eqref{eq:bma}:
\begin{enumerate}
\item \emph{Cumulative statistics (Layer~1):} A causal attention head can be configured to produce uniform weights over the causal prefix, computing running averages of $[\mathrm{vec}(x_i x_i^\top);\, x_i y_i;\, y_i^2]$. A subsequent MLP, given access to the position index via positional encoding, approximates multiplication by $t$ to recover cumulative prefix sums $P_j^{xx}, P_j^{xy}, P_j^{yy}$.
\item \emph{Segment isolation (Layer~2):} At the prediction position $t$, multi-head attention retrieves prefix sums at each candidate split $k \in \mathcal{K}_t$, and an MLP forms segment statistics via subtraction: $A_k(t) = P_{t-1}^{xx} - P_k^{xx}$, and similarly for $b_k$ and $c_k$.
\item \emph{Posterior computation and prediction (Layers~3--4):} MLPs approximate the posterior parameters $\mu_k, \Sigma_k$, the log-marginal likelihoods $\log p(Z_{<t} \mid k)$, and the posterior weights $\alpha_k(t)$, then form the weighted prediction~\eqref{eq:bma}.
\end{enumerate}
All computation from Layer~2 onward occurs at the prediction token, so causality is preserved without backward information flow.

\begin{theorem}[Approximation of Change-Point BMA]
\label{thm:main}
Consider the piecewise-linear model~\eqref{eq:model} with bounded inputs and the Gaussian prior and noise model specified above. For every $\epsilon > 0$, there exists a causal transformer $T_\epsilon$ such that
\begin{equation}
\bigl| T_\epsilon(Z_{<t}, x_t) - \hat{y}_t^{\mathrm{BMA}} \bigr| \leq \epsilon
\label{eq:approx}
\end{equation}
for all $t = 1, \dots, N$ and all input sequences satisfying the boundedness assumptions. The number of layers is constant for fixed $\epsilon$, and the number of attention heads and hidden dimension scale with $|\mathcal{K}_t|$ and $d$.
\end{theorem}

\begin{proof}[Proof sketch]
The proof proceeds in three parts: showing that the transformer can causally assemble sufficient statistics, that these statistics determine the posterior predictive means $m_k(t)$ and weights $\alpha_k(t)$ via continuous maps, and that the resulting BMA predictor can be approximated.
 
\emph{Part 1: Causal assembly of sufficient statistics.} At each context position $j$, the token embedding includes $\mathrm{vec}(x_j x_j^\top)$, $x_j y_j$, and $y_j^2$. A single attention head with zero query and key projections produces uniform weights $\alpha_{ti} = 1/t$, yielding the running average $\bar{S}_t = \frac{1}{t}\sum_{i=1}^t S_i^{\mathrm{raw}}$ where $S_i^{\mathrm{raw}} = [\mathrm{vec}(x_i x_i^\top);\, x_i y_i;\, y_i^2]$. A subsequent MLP, given access to the position index $t$ through the positional encoding, approximates the product $t \cdot \bar{S}_t$ to recover cumulative prefix sums $P_j^{xx} = \sum_{i=1}^j x_i x_i^\top$, $P_j^{xy} = \sum_{i=1}^j x_i y_i$, $P_j^{yy} = \sum_{i=1}^j y_i^2$. This is the first role of PE in the construction: without access to the position index, the MLP cannot convert running averages into cumulative sums.

At the prediction position $t$, the transformer must retrieve these prefix sums at each candidate change-point index $k \in \mathcal{K}_t$. The candidate set $\mathcal{K}_t$ is determined by the side information communicated through positional encoding. When the PE encodes only the support $[L, U]$ of the change point, the transformer configures $|\mathcal{K}_t| = U - L + 1$ attention heads, each with query-key projections over the positional features designed to produce high attention weight at the corresponding candidate index $k$. When no side information is provided, $\mathcal{K}_t = \{1, \dots, t{-}1\}$ and $O(t)$ heads are required in principle. When the change point is known exactly, $|\mathcal{K}_t| = 1$ and a single head suffices. This is the second and primary role of PE: it determines the size of the candidate set and thereby controls the number of attention heads needed for retrieval. An MLP then computes segment statistics via subtraction: $A_k(t) = P_{t-1}^{xx} - P_k^{xx}$, $b_k(t) = P_{t-1}^{xy} - P_k^{xy}$, $c_k(t) = P_{t-1}^{yy} - P_k^{yy}$. After this stage, the prediction token's representation contains $(A_k, b_k, c_k)$ for every candidate $k \in \mathcal{K}_t$.
 
\emph{Part 2: From sufficient statistics to $m_k(t)$ and $\alpha_k(t)$.} After Part~1, the prediction token's residual stream contains the segment statistics $(A_k, b_k, c_k)$ for every candidate $k \in \mathcal{K}_t$, along with the query input $x_t$. All remaining computation is local to this single token and is performed by MLP layers.

First, the predictive means. Under the Gaussian prior and noise model, the posterior over the post-change weight given hypothesis $k$ is $w \mid Z_{<t}, k \sim \mathcal{N}(\mu_k, \Sigma_k)$ with $\Sigma_k = (\lambda I + \sigma^{-2} A_k)^{-1}$ and $\mu_k = \sigma^{-2} \Sigma_k \, b_k$. The predictive mean under hypothesis $k$ is $m_k(t) = x_t^\top \mu_k$. Computing $\mu_k$ requires inverting $\lambda I + \sigma^{-2} A_k$, which is continuous on the bounded domain since $\lambda > 0$ ensures uniform positive definiteness. An MLP can therefore approximate the map $(A_k, b_k, x_t) \mapsto m_k(t)$ for each candidate.

Second, the posterior weights. The weight of hypothesis $k$ is $\alpha_k(t) \propto \pi_k \, p(Z_{<t} \mid k)$, where the log-marginal likelihood under the Gaussian model is:
\begin{equation} 
\log p(Z_{<t} \mid k) = -\tfrac{n_k}{2}\log(2\pi\sigma^2) - \tfrac{1}{2}\log\det(\lambda^{-1} A_k + I) \\
- \tfrac{1}{2\sigma^2}\bigl(c_k - \sigma^{-2} b_k^\top \Sigma_k \, b_k\bigr)
\label{eq:logml}
\end{equation}

up to pre-change terms that cancel in the normalization. Here $n_k = |\{i : k < i < t\}|$ is the post-change sample size under hypothesis $k$. Every term in~\eqref{eq:logml}---including the log-determinant and the quadratic form $b_k^\top \Sigma_k \, b_k$---is a continuous function of the sufficient statistics $(A_k, b_k, c_k)$ on the bounded domain. Note that $c_k$ is essential: without it, one can compute the predictive means $m_k$ but not the marginal likelihoods needed to weight them.

Finally, the MLP computes the posterior weights $\alpha_k(t) = \mathrm{softmax}_k(\log p(Z_{<t} \mid k))$ and forms the weighted prediction $\hat{y}_t^{\mathrm{BMA}} = \sum_k \alpha_k(t) \, m_k(t)$.
 
\emph{Part 3: Approximation.} The entire computation from sufficient statistics to final prediction is a composition of continuous maps on a compact domain: $(A_k, b_k, x_t) \mapsto m_k(t)$, then $(A_k, b_k, c_k) \mapsto \log p(Z_{<t} \mid k)$, then softmax normalization to obtain $\alpha_k(t)$, and finally the weighted sum $\hat{y}_t^{\mathrm{BMA}} = \sum_k \alpha_k(t)\, m_k(t)$. Since all $|\mathcal{K}_t|$ candidates reside in the residual stream of a single token, this composition is executed entirely at the prediction position. By universal approximation results for multilayer networks on compact domains, the MLP blocks can approximate this map to within $\epsilon$.
\end{proof}

\begin{remark}[Detection as a consequence]
\label{rem:detection}
Theorem~\ref{thm:main} is an approximation guarantee for BMA, not a change-point recovery result. However, when the task separation $\Delta = \|w_1 - w_2\|$ is large relative to noise, the BMA posterior concentrates on the true change point, so that~\eqref{eq:bma} reduces to prediction using only post-change data. Under the Gaussian model, incorrect hypotheses are exponentially downweighted by their excess log-likelihood.
\end{remark}

\begin{corollary}[Known Change Point]
\label{cor:known}
If $n_1^*$ is known to the transformer via positional encoding, the candidate set reduces to $\mathcal{K}_t = \{n_1^*\}$. A causal transformer with $O(1)$ layers (for fixed $\epsilon$), hidden dimension $O(d^2)$, and $O(1)$ attention heads achieves
\[
\bigl| T_\epsilon(Z_{<t}, x_t) - x_t^\top \mu_{n_1^*} \bigr| \leq \epsilon, \quad t > n_1^*,
\]
where $\mu_{n_1^*}$ is the Bayesian posterior mean using only post-change data. The construction simplifies in two ways. First, a single attention head with positional masking suffices to accumulate post-change statistics, eliminating the $O(|\mathcal{K}_t|)$-head retrieval layer and the MLP layers devoted to log-marginal likelihood evaluation and posterior weighting. Second, the prediction target itself improves: instead of approximating BMA, which hedges across hypotheses, the transformer directly approximates the oracle posterior mean under the true segmentation. 
\end{corollary}

This corollary corresponds to the oracle baseline in our experiments and highlights the complexity--information tradeoff: knowing the change point reduces $|\mathcal{K}_t|$ to $1$, reducing the required number of attention heads from $O(R)$ to $O(1)$.

\section{Synthetic Experiments}
\label{sec:experiments}

We evaluate our theoretical framework in two synthetic settings: piecewise-linear regression and piecewise-linear dynamical systems. Both settings are designed to isolate the core challenge of in-context regime adaptation in a controlled setting where we can calculate the ideal behavior. In addition, many recent works have shown that when transformers are trained on large volumes of synthetic tasks, they can generalize well to real-world tasks without retraining \cite{tabpfn, sgnns2025, timepfn, sgnntheory}. A sequence is generated by one linear process up to an unknown change point and by a different linear process afterward. The transformer must detect the transition from context and quickly adapt its predictions to the new regime.

Across both settings, we train GPT-2 style causal transformers under stochastic change-point locations and compare them to ideal baselines matched to the information available to the model. Our goal is not only to test whether transformers benefit from additional change-point information, but also whether they approach the best possible in-context strategy under each information regime.

\subsection{Encoding Change-Point Information}
\label{sec:pe}

We use positional encoding (PE) to provide side information about the change point. This lets us vary the amount of information available to the model without changing the architecture or training objective. We consider three PE schemes---no PE, sinusoidal PE, and linear PE---under four information levels:
\begin{enumerate}
    \item \textbf{No information:} the model receives no explicit signal about the change point.
    \item \textbf{Support known:} the model is told only the interval in which the change point lies.
    \item \textbf{Known in advance:} the exact change point is encoded before the transition occurs.
    \item \textbf{Known afterward:} the model is informed only after the transition has occurred.
\end{enumerate}
This hierarchy allows us to study how side information changes the difficulty of in-context adaptation.

\subsection{Piecewise-Linear Regression}
\label{sec:linreg}

\paragraph{Setup}
We first consider a piecewise-linear regression task. Each prompt consists of $N=30$ data-label pairs generated by one linear predictor before the change point and a different linear predictor afterward. During training, the change point is drawn uniformly from $\{10,\dots,20\}$. At test time, we fix the change point to $n_1=12$ so that performance cannot be explained by a simple bias toward the midpoint of the training distribution. Note that because the change point is always drawn from $\{10, \dots, 20\}$ during training, the no-information variant can implicitly learn this support from the training distribution, making it functionally equivalent to the support-known variant in this experiment.

\paragraph{Baselines}
We compare transformer performance to three ideal baselines.
\begin{itemize}
    \item \textbf{Oracle ridge regression:} for settings where the model knows the change point, we compare against an oracle that knows the regime boundaries and performs ridge regression using only samples from the active regime.
    \item \textbf{Transfer ridge regression:} for the correlated-task setting where $w_2 = -w_1 + \varepsilon \eta$ with $\eta \sim \mathcal{N}(0, I)$, simply discarding pre-change data is suboptimal because the task~1 samples carry information about $w_2$. This baseline first estimates $w_1$ via ridge regression on the pre-change data, obtaining the posterior $w_1 \mid X_1, Y_1 \sim \mathcal{N}(\hat{w}_1, \Sigma_1)$. Propagating through the known relationship gives an informative prior $w_2 \sim \mathcal{N}(-\hat{w}_1, \Sigma_1 + \varepsilon^2 I)$ on the post-change weights, which is then updated via ridge regression as task~2 samples arrive.
    \item \textbf{Bayesian model averaging (BMA):} for settings in which the model is not given the change point, we compare against a Bayesian baseline that maintains a posterior over all candidate change-point locations and combines the corresponding ridge regression predictors according to their posterior weights.
\end{itemize}
These baselines represent the ideal prediction strategies with and without change-point information, respectively.

\paragraph{Evaluation}
We report mean squared error (MSE) at each time step, averaged over 5,000 test trajectories. This lets us directly measure detection lag, post-change adaptation, and the benefit of side information.

\paragraph{Results}
Figure~\ref{fig:linreg} shows the per-step MSE for the piecewise-linear regression task, averaged over 5,000 test trajectories with the change point fixed at $n_1^* = 12$. Panel~(a) compares all four information levels on a log scale. Before the change point, the known-in-advance variant achieves lower error than the other three variants. This is because the uninformed variants must begin hedging before the change point to avoid catastrophic misprediction if a regime shift occurs, sacrificing pre-change accuracy for robustness. The known-in-advance variant, which knows the transition will not occur until $t = 12$, can commit fully to the pre-change task and make the best possible prediction at every step. At the transition, all variants spike, but with different magnitudes. The known-in-advance variant returns to roughly baseline error levels: it knows the old regime has ended and does not mispredict with stale dynamics, but has no observations from the new regime yet and must predict from an uninformed prior. The other three variants, still relying on the old dynamics, produce catastrophic errors that exceed baseline levels by almost an order of magnitude. After one step of this catastrophic misprediction, the known-afterward variant receives the change-point signal and immediately matches the known-in-advance error, since both now condition on the correct segmentation. The no-information and support-known variants, lacking exact change-point timing, require several additional post-change observations to implicitly detect the regime shift and downweight obsolete data. The similar recovery profiles of the no-information and support-known variants are expected: since the training distribution always draws change points from $\{10, \dots, 20\}$, the no-information model implicitly learns this support, rendering the two variants functionally equivalent.

Panels~(b)--(d) compare individual transformer variants to their corresponding ideal baselines. The known-in-advance transformer closely matches the oracle least-squares baseline (panel~b), which knows the true segmentation and fits only on data from the active regime. This validates Corollary~\ref{cor:known} in the regression setting: when the change point is provided, the transformer learns the oracle strategy.

Panel~(c) tests a more challenging transfer variant in which the post-change weights are anticorrelated with the pre-change weights: $w_2 = -w_1 + \varepsilon$ for small $\varepsilon$. Here the ideal strategy is not to discard pre-change data, but to leverage it through a Bayesian transfer prior that reflects the known relationship between regimes \cite{gatingisweighting}. The known-in-advance transformer closely matches this Bayesian transfer baseline, demonstrating that the model learns to exploit inter-regime structure rather than simply discarding pre-change observations. This goes beyond our theoretical construction, which assumes independent priors on $w_1$ and $w_2$, and suggests that trained transformers can discover richer adaptation strategies than the ones we prove exist.

Panel~(d) shows that the support-known transformer closely tracks the BMA baseline, which maintains a posterior over all candidate change-point locations within the known support and averages predictions accordingly. This confirms that when partial information narrows the candidate set $\mathcal{K}_t$ but does not resolve the change point exactly, transformers learn to perform the corresponding hypothesis averaging in-context, consistent with Theorem~\ref{thm:main}.

\subsection{Piecewise-Linear Dynamical Systems}
\label{sec:lds}

\paragraph{Setup}
We next consider a piecewise-linear dynamical system (LDS) setting. Here the prompt is a trajectory of $N=40$ states, generated by one stable linear dynamical system before the change point and another afterward. The two system matrices are drawn independently for each trajectory. During training, the change point is sampled uniformly from $\{15,\dots,25\}$, and at test time we fix it to $n_1=17$.

This setting is more challenging than the regression setting because the observations are temporally dependent. The model must not only infer that a regime change has occurred, but also rapidly identify the new system dynamics from a small number of post-change transitions.

\paragraph{Baselines}
We again compare against two ideal baselines, now adapted to online system identification. The LDS prediction task $x_{t+1} = A x_t + \varepsilon_t$ can be viewed as $d$ parallel linear regressions, one per row of $A$, with $x_t$ as the shared covariate. We place a Gaussian prior on each row, $a_r \sim \mathcal{N}(0, \tau^2 I_d)$, matching the distribution used to generate the system matrices. The posterior mean of $A$ given a segment of transitions is then the Bayesian ridge regression estimate with regularization $\lambda = 1/\tau^2$, the matrix analogue of the posterior mean $\mu_k = \sigma^{-2} \Sigma_k b_k$ in the regression setting.

\begin{itemize}
    \item \textbf{Oracle segmented ridge regression:} for informed settings, we compare against a baseline that knows the true change point but not the system matrices. At each time $t$, it accumulates transitions from the correct regime only and predicts with the corresponding ridge estimate of $A$. This is the analogue of the oracle regression baseline.
    \item \textbf{Bayesian changepoint averaging:} for uninformed settings, we compare against a baseline that maintains a posterior over candidate change points $k \in \{15, \dots, 25\}$, initialized uniformly. Each hypothesis defines its own segmentation and ridge estimate of $A$; the mixture prediction averages these, weighted by sequentially updated posterior probabilities. Predictive uncertainty is incorporated so that hypotheses with few segment observations are penalized less for large errors, preventing premature elimination early in a segment.
\end{itemize}
Both baselines are Bayes-optimal under the known prior and noise model. Even the informed oracle must estimate the system matrices from data, and the ridge regularization is essential in the post-change period when the number of available transitions may be smaller than the state dimension.

\paragraph{Evaluation}
As in the regression experiments, we report per-step MSE averaged over 5,000 test trajectories. In this setting, we are especially interested in the post-change transient, since adaptation requires identifying a new dynamical system rather than simply switching to a new regression weight vector.

\begin{figure*}[htbp]
    \centering
    \includegraphics[width=0.95\textwidth]{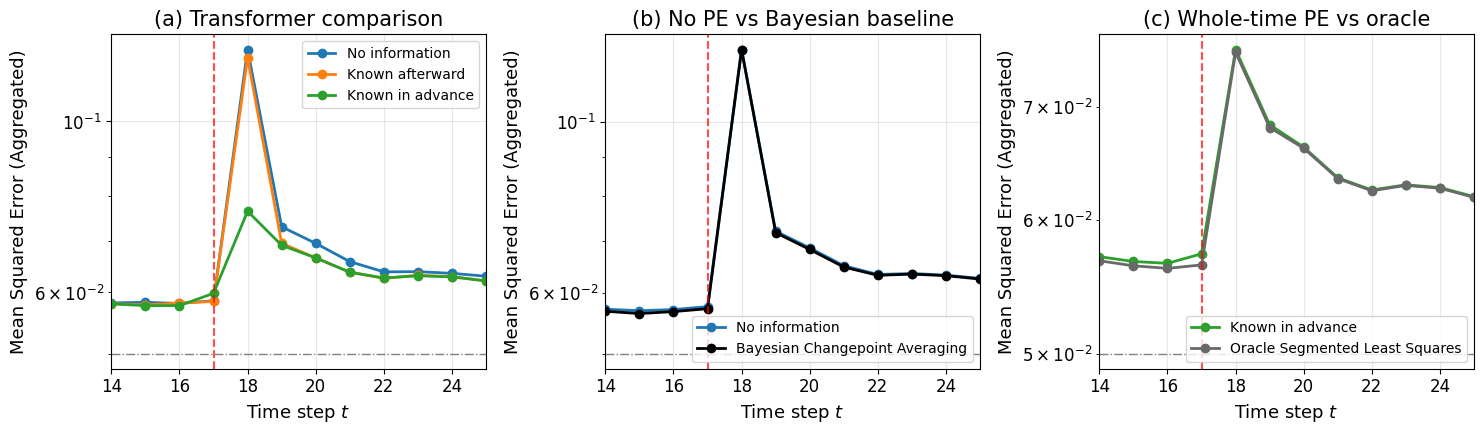}
    \caption{Piecewise-linear dynamical system results. Mean squared error (aggregated over 5,000 test trajectories) as a function of time step $t$, with the change point at $n_1^* = 17$ (dashed red line). \textbf{(a)}~Comparison of transformer variants: no positional encoding, positional encoding revealed after the change point, and positional encoding available for the whole sequence. All variants spike at the regime transition; the whole-time variant recovers fastest. \textbf{(b)}~The no-PE transformer closely tracks the Bayesian changepoint averaging baseline, confirming that the uninformed model learns to implement approximate posterior averaging over change-point hypotheses in-context. \textbf{(c)}~The whole-time PE transformer closely matches the oracle segmented least-squares baseline, which knows the true change point but must still estimate system dynamics from data. This validates the known-change-point construction (Corollary~\ref{cor:known}).}
    \label{fig:lds}
\end{figure*}

\paragraph{Results}
Figure~\ref{fig:lds} shows per-step MSE averaged over 5,000 test trajectories with $n_1^*=17$. All variants spike at the change point, but the whole-time PE variant recovers fastest: its error reflects only prior uncertainty over the new system, whereas the no-PE and known-after variants initially mispredict by extrapolating stale dynamics. The known-after variant matches whole-time PE as soon as it receives the change-point signal. The no-PE variant requires several additional steps to implicitly detect the shift. Panels (b) and (c) show that the no-PE transformer closely tracks Bayesian changepoint averaging, while the whole-time PE transformer matches the oracle segmented least-squares baseline. These results confirm that the complexity-information tradeoff from our theory carries over to temporally dependent settings, and that trained transformers learn to match the optimal baseline corresponding to their information level.

\section{Real-world experiments}

\subsection{Forecasting Infectious Diseases with Policy Changes}

Recent advances in infectious disease forecasting have introduced foundation models that aim to generalize across diseases, regions, and surveillance settings without retraining. Models such as Mantis \cite{mantis} learn from large-scale synthetic or historical data and can produce accurate forecasts across a wide range of infectious diseases. More general time series foundation models have also been applied in epidemiological settings with promising results \cite{timeseriesfmdisease}. However, a key limitation remains. When the underlying disease dynamics undergo a regime shift---for example, due to a change in government policy or intervention strategy---these models often fail to rapidly adapt, as they assume continuity in the data-generating process.

From a control perspective, policy changes act as external interventions that alter the system dynamics governing disease spread. After such a change, past data may no longer be predictive, and the model must rapidly adapt by reweighting or discarding obsolete information. In practice, we often have advance knowledge that a policy change will occur, even if its exact effect is unknown. This raises a key question: how should this change-point information be communicated to a foundation model to adapt its forecasts without retraining?

To study this, we use infectious disease data from Thailand spanning 1980--2021, a period with multiple government-driven regime shifts. This provides a realistic testbed for non-stationarity in epidemiological time series. We evaluate whether encoding change-point information enables a pretrained forecasting model to adjust its predictions at inference time, translating the principles from our synthetic setting to real-world, policy-driven dynamics.

\paragraph{Setup}
We construct a collection of disease-policy episodes from Thai infectious disease data, where each episode consists of a window around a known policy change.. In each evaluation run, we perform leave-one-sequence-out evaluation: one episode is designated as the target sequence, while the remaining episodes are used as context. Context sequences are not restricted to the same disease family, allowing the model to draw on regime-shift patterns across diseases.

Following the tabular forecasting formulation of TabPFN-TS \cite{tabpfnts}, we convert each episode into a tabular dataset with one row per time step. Each row includes the relative time index within the window, z-score normalized observation value, a binary direction indicator specifying whether the policy change is expected to produce an upward or downward shift, and a positional-encoding feature. We compare three variants of this final feature: no positional encoding, linear positional encoding based on distance to the change point, and sinusoidal positional encoding.

Each episode has length 30. For the target sequence, the change point is fixed at $t=12$. For context sequences, the change point is sampled uniformly from $\{10,\dots,20\}$, so that the model cannot rely on a fixed transition time and must instead use the positional signal when it is available. At evaluation time $t$, we provide the model with all rows from the context episodes together with the observed prefix of the target episode up to time $t$.

\paragraph{Evaluation}
We evaluate all time points $t$ in the target sequence. At each $t$, the model produces recursive forecasts for the next three observations, and we compute the mean absolute error over this horizon. We then average these errors across all target episodes. Our primary comparison is between the no-PE, linear-PE, and sinusoidal-PE settings. Improvements from the positional-encoding variants indicate that explicitly providing change-point information helps a pretrained foundation model adapt to policy-driven regime shifts without retraining.

\begin{figure}[htbp]
    \centering
    \includegraphics[width=0.49\textwidth]{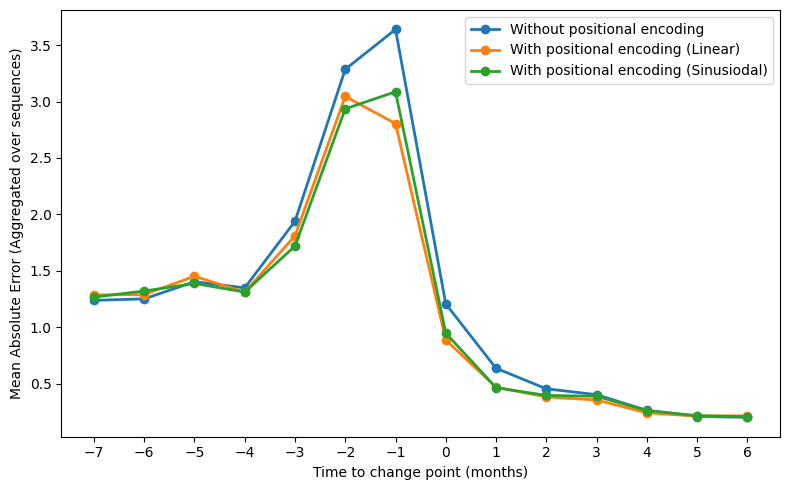}
    \caption{Mean absolute error for infectious disease forecasting as a function of forecast origin time relative to the policy change point. Each forecast is a recursive three-step-ahead prediction, so forecasts originating shortly before the change point span the regime transition. Both linear and sinusoidal positional encodings reduce MAE by approximately 25\% at the forecast origins most affected by the regime shift, compared to the no-PE baseline.}
    \label{fig:disease}
\end{figure}
\paragraph{Results}
Figure~\ref{fig:disease} shows the MAE as a function of forecast origin time relative to the change point. The largest performance gap between PE variants and the no-PE baseline occurs at origin times shortly before the change point, where the best PE variant reduces error by approximately 25\%. This timing is explained by the multi-step forecast horizon: because each prediction covers three steps ahead, forecasts originating just before the change point must predict across the regime transition. Without positional information, the model extrapolates pre-change dynamics into the new regime; with PE, the model can anticipate the upcoming shift and adjust its predictions accordingly. After the change point, all methods improve as post-change observations accumulate in the conditioning window, and the advantage of positional encoding gradually diminishes as the model can infer the new regime directly from recent data. Overall, these results demonstrate that encoding change-point information enables a pretrained foundation model to adapt to policy-driven regime shifts at inference time, with the largest gains precisely where the forecasting task spans a regime boundary.

\subsection{Forecasting Financial Series with Monetary Policy Events}

Scheduled monetary policy announcements can introduce abrupt changes in financial time series, making forecasting more difficult around event dates. In particular, Federal Open Market Committee (FOMC) announcements can influence both short-term interest rates and broader asset-market behavior \cite{kuttner2001,bernanke2005}. Because the timing of these events is known in advance, they provide a natural setting for testing whether explicit event-time information helps a pretrained model improve its forecasts.

\paragraph{Setup}
We study two daily financial series from FRED: the effective federal funds rate and the S\&P 500 index \cite{fred_dff,fred_sp500}. These are paired with FOMC statement dates from the Federal Reserve's historical records \cite{fomc_historical}. For each business day $t$, the prediction target is the next-day value, yielding a direct one-step-ahead forecasting task.

Each observation is represented in tabular form using the current series value together with standard time-series covariates: lags 1--10, the first difference, a 5-day rolling mean, and a 5-day rolling standard deviation. On top of these baseline features, we compare three event-aware variants. The no-PE baseline uses only the time-series covariates. The linear-PE variant adds the signed business-day distance to the nearest FOMC event, so that negative values indicate days before the event and positive values indicate days after it. The sinusoidal-PE variant instead adds sine and cosine encodings of the signed distance.

Evaluation is performed with walk-forward forecasting. Each training window contains roughly three years of business-day observations and is capped at that size, while testing is carried out on successive 21-business-day blocks. For each split, a TabPFN regressor is trained on the current window and evaluated on the next block.

\paragraph{Evaluation}
We report overall MAE across all walk-forward predictions. Our primary comparison is between the no-PE, linear-PE, and sinusoidal-PE variants for each financial series. Improvements from the positional-encoding models indicate that supplying explicit information about policy-event timing helps the pretrained model adapt to event-related changes in financial dynamics.

\paragraph{Results}
Figure~\ref{fig:fedfunds} shows the same overall ordering across the three variants. In both cases, linear positional encoding achieves the lowest MAE, sinusoidal positional encoding is slightly higher, and the no-PE baseline has the largest error.

\begin{figure}[htbp]
    \centering
    \includegraphics[width=0.45\textwidth]{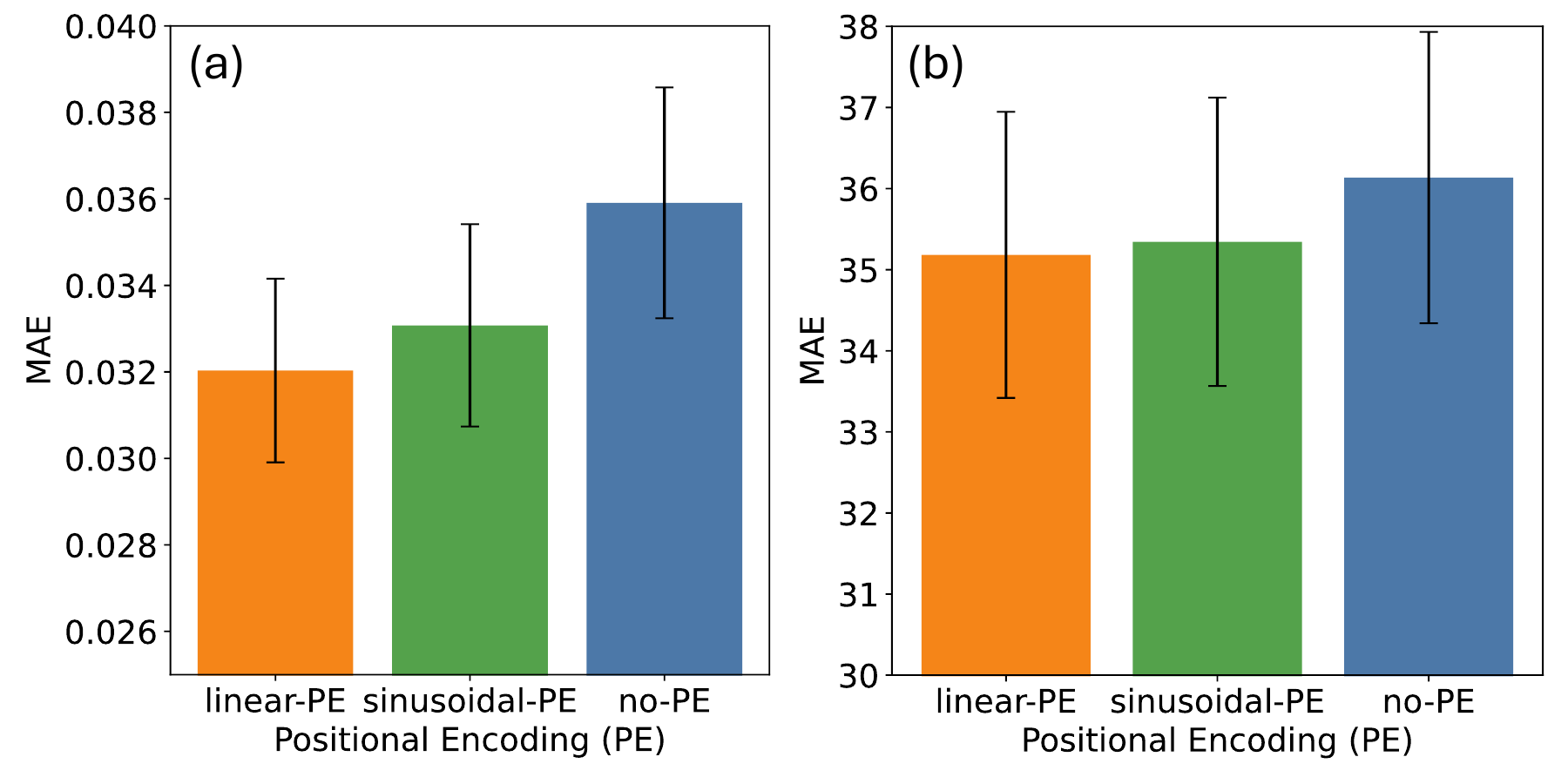}
\caption{Mean absolute error for federal funds rate (a) and S\&P 500 index (b) forecasting and under three feature settings: no positional encoding, linear positional encoding based on signed business-day distance to the nearest Federal Open Market Committee (FOMC) statement day, and sinusoidal positional encoding. Error bars show the standard error.}

    \label{fig:fedfunds}
\end{figure}

\section{Discussion}
\label{sec:discussion}

We have shown that transformers can approximate the Bayesian model-averaged predictor for in-context change-point detection, with model complexity governed by the amount of side information available about the change point. Our synthetic experiments confirm that trained transformers match the performance of optimal baselines---BMA when uninformed and oracle least squares when informed---across both regression and dynamical system settings. Our real-world experiments demonstrate that positional encoding of change-point information improves the forecasts of a pretrained foundation model on infectious disease and financial data without any retraining.

\paragraph{Limitations} Our theoretical analysis assumes a single change point and a linear task. Extension to multiple change points would require increasing the required depth or width. Extension to nonlinear tasks would require different sufficient statistics and is an open question. However, our real-world experiments suggest these limitations may be more theoretical than practical: both real-world tasks involve multiple regime shifts and nonlinear dynamics, yet the same strategy yields consistent improvements. This suggests the core mechanism transfers beyond the setting our theory covers. Our construction demonstrates existence of suitable transformer weights but does not address whether gradient descent finds them during training; the synthetic experiments suggest it does, but a formal analysis remains future work.

\paragraph{Broader implications} Foundation models for time series and tabular prediction are increasingly deployed in settings where regime shifts are common, yet their architectures and training procedures are designed for stationary prompts. Our results suggest a simple, practical intervention: encoding change-point information through positional features at inference time, without retraining. More broadly, the framework demonstrates a natural connection between classical sequential detection and modern in-context learning. The BMA predictor our transformer approximates is closely related to Bayesian online change-point detection \cite{adams2007}, and the known-change-point construction mirrors adaptive control strategies that discard pre-fault data after a detected switch \cite{basseville1993detection}. As foundation models are deployed for forecasting, control, and decision-making in non-stationary environments, the extensive toolkit developed by the control and detection communities becomes directly relevant, not as competing methods, but as the algorithmic targets that these models must learn to implement in-context.

\bibliographystyle{unsrt}
\bibliography{references}

@misc{brown2020,
      title={Language Models are Few-Shot Learners}, 
      author={Tom B. Brown and others},
      year={2020},
      eprint={2005.14165},
      archivePrefix={arXiv},
      primaryClass={cs.CL},
      url={https://arxiv.org/abs/2005.14165}, 
}

@misc{garg2022,
      title={What Can Transformers Learn In-Context? A Case Study of Simple Function Classes}, 
      author={Shivam Garg and others},
      year={2023},
      eprint={2208.01066},
      archivePrefix={arXiv},
      primaryClass={cs.CL},
      url={https://arxiv.org/abs/2208.01066}, 
}

@misc{vonoswald2023,
      title={Transformers learn in-context by gradient descent}, 
      author={Johannes von Oswald and others},
      year={2023},
      eprint={2212.07677},
      archivePrefix={arXiv},
      primaryClass={cs.LG},
      url={https://arxiv.org/abs/2212.07677}, 
}

@misc{akyurek2023,
      title={What learning algorithm is in-context learning? Investigations with linear models}, 
      author={Ekin Akyürek and others},
      year={2023},
      eprint={2211.15661},
      archivePrefix={arXiv},
      primaryClass={cs.LG},
      url={https://arxiv.org/abs/2211.15661}, 
}

@article{hamilton1989new,
  title={A New Approach to the Economic Analysis of Nonstationary Time Series and the Business Cycle},
  author={Hamilton, James D.},
  journal={Econometrica},
  volume={57},
  number={2},
  pages={357--384},
  year={1989},
}

@techreport{ang2011regime,
  title={Regime Changes and Financial Markets},
  author={Ang, Andrew and Timmermann, Allan},
  year={2011},
  month={June},
  institution={National Bureau of Economic Research},
  type={Working Paper},
}

@article{page1954,
  title={Continuous inspection schemes},
  author={Page, E. S.},
  journal={Biometrika},
  volume={41},
  year={1954},
  publisher={JSTOR}
}

@article{lai1995,
  title={Sequential changepoint detection in quality control and dynamical systems},
  author={Lai, Tze Leung},
  journal={Journal of the Royal Statistical Society: Series B},
  volume={57},
  number={4},
  year={1995},
}

@misc{adams2007,
      title={Bayesian Online Changepoint Detection}, 
      author={Ryan Prescott Adams and David J. C. MacKay},
      year={2007},
      eprint={0710.3742},
      archivePrefix={arXiv},
      primaryClass={stat.ML},
      url={https://arxiv.org/abs/0710.3742}, 
}

@misc{arxiv_version,
      title={Transformers as Adaptive Estimators: In-Context Learning Under Regime Change}, 
      author={Carson Dudley and Yutong Bi and Xiaofeng Liu and Samet Oymak},
      year={2026},
      note={Technical Report}
}

@misc{chronos,
    title = {Chronos: Transformer-Based Language Models for Time-Series Forecasting},
    author = {Abdul Fatir Ansari and others},
    year = {2024},
    eprint = {2403.07815},
    archivePrefix = {arXiv},
    primaryClass = {cs.LG}
}

@article{tabpfn,
  title={Accurate predictions on small data with a tabular foundation model},
  author={Hollmann, Noah and M{\"u}ller, Samuel and others},
  journal={Nature},
  volume={637},
  pages={319--326},
  year={2025},
}

@misc{timesfm,
      title={A decoder-only foundation model for time-series forecasting}, 
      author={Abhimanyu Das and Weihao Kong and Rajat Sen and Yichen Zhou},
      year={2024},
      eprint={2310.10688},
      archivePrefix={arXiv},
      primaryClass={cs.CL},
      url={https://arxiv.org/abs/2310.10688}, 
}

@article{sgnntheory,
  title={Learning From Simulators: A Theory of Simulation-Grounded Learning},
  author={Carson Dudley and Marisa Eisenberg},
  journal={arXiv preprint arXiv:2509.18990},
  year={2025},
  url={https://arxiv.org/abs/2509.18990}
}

@article{dudley2025sparse,
  title={From Sparse Data to Smart Decisions: Region-Specific Policy Evaluation via Simulation},
  author={Dudley, Carson and others},
  journal={medRxiv},
  year={2025},
  doi={10.64898/2025.12.05.25341712},
  url={https://doi.org/10.64898/2025.12.05.25341712}
}

@misc{gatingisweighting,
      title={Gating is Weighting: Understanding Gated Linear Attention through In-context Learning}, 
      author={Yingcong Li and others},
      year={2025},
      eprint={2504.04308},
      archivePrefix={arXiv},
      primaryClass={cs.LG},
      url={https://arxiv.org/abs/2504.04308}, 
}

@ARTICLE{willisky1976,
  author={Willsky, A. and Jones, H.},
  journal={IEEE Transactions on Automatic Control}, 
  title={A generalized likelihood ratio approach to the detection and estimation of jumps in linear systems}, 
  year={1976},
  volume={21},
  number={1},
  pages={108-112},
  doi={10.1109/TAC.1976.1101146}}

@book{basseville1993detection,
  title     = {Detection of Abrupt Changes: Theory and Application},
  author    = {Basseville, Mich\`{e}le and Nikiforov, Igor V.},
  year      = {1993},
}

@inproceedings{timepfn,
    title={Time{PFN}: Effective Multivariate Time Series Forecasting with Synthetic Data},
    author={Ege Onur Taga and M. Emrullah Ildiz and Samet Oymak},
    booktitle={Proceedings of the AAAI Conference on Artificial Intelligence},
    year={2025},
}

@article{timeseriesfmdisease,
  title={Foundation time series models for forecasting and policy evaluation in infectious disease epidemics},
  author={Kalahasti, Suprabhath and others},
  journal={medRxiv},
  year={2025},
  month={February},
  doi={10.1101/2025.02.24.25322795},
}

@misc{zhang2023trained,
      title={Trained Transformers Learn Linear Models In-Context}, 
      author={Ruiqi Zhang and Spencer Frei and Peter L. Bartlett},
      year={2023},
      eprint={2306.09927},
      archivePrefix={arXiv},
      primaryClass={stat.ML},
      url={https://arxiv.org/abs/2306.09927}, 
}

@article{hoeting1999bma,
  title={Bayesian model averaging: A tutorial},
  author={Hoeting, Jennifer A. and Madigan, David and Raftery, Adrian E. and Volinsky, Chris T.},
  journal={Statistical Science},
  volume={14},
  number={4},
  pages={382--417},
  year={1999},
  publisher={Institute of Mathematical Statistics}
}

@article{min1993bayesian,
  title={Bayesian and non-{B}ayesian methods for combining models and forecasts with applications to forecasting international growth rates},
  author={Min, Chung-ki and Zellner, Arnold},
  journal={Journal of Econometrics},
  volume={56},
  number={1-2},
  pages={89--118},
  year={1993},
  publisher={Elsevier}
}

@misc{zhang2023icl_bma,
      title={What and How does In-Context Learning Learn? Bayesian Model Averaging, Parameterization, and Generalization}, 
      author={Yufeng Zhang and Fengzhuo Zhang and Zhuoran Yang and Zhaoran Wang},
      year={2023},
      eprint={2305.19420},
      archivePrefix={arXiv},
      primaryClass={stat.ML},
      url={https://arxiv.org/abs/2305.19420}, 
}

@misc{vaswani2023attentionneed,
      title={Attention Is All You Need}, 
      author={Ashish Vaswani and Noam Shazeer and others},
      year={2023},
      eprint={1706.03762},
      archivePrefix={arXiv},
      primaryClass={cs.CL},
      url={https://arxiv.org/abs/1706.03762}, 
}

@misc{li2023,
      title={Transformers as Algorithms: Generalization and Stability in In-context Learning}, 
      author={Yingcong Li and M. Emrullah Ildiz and Dimitris Papailiopoulos and Samet Oymak},
      year={2023},
      eprint={2301.07067},
      archivePrefix={arXiv},
      primaryClass={cs.LG},
      url={https://arxiv.org/abs/2301.07067}, 
}

@article{sgnns2025,
  title={Simulation as Supervision: Mechanistic Pretraining for Scientific Discovery},
  author={Dudley, Carson and others},
  journal={arXiv preprint arXiv:2507.08977},
  year={2025},
  url={https://arxiv.org/abs/2507.08977}
}

@misc{tabpfnts,
      title={From Tables to Time: Extending TabPFN-v2 to Time Series Forecasting}, 
      author={Shi Bin Hoo and Samuel Müller and David Salinas and Frank Hutter},
      year={2026},
      eprint={2501.02945},
      archivePrefix={arXiv},
      primaryClass={cs.LG},
      url={https://arxiv.org/abs/2501.02945}, 
}

@article{mantis,
    title={Mantis: A Foundation Model for Mechanistic Disease Forecasting}, 
    author={Carson Dudley and others},
    year={2025},
    journal={arXiv preprint arXiv:2508.12260},
    url={https://arxiv.org/abs/2508.12260},
    volume={n/a}
}

@misc{fred_dff,
  author = {{Board of Governors of the Federal Reserve System (US)}},
  title = {Federal Funds Effective Rate [DFF]},
  year = {2026},
  howpublished = {\url{https://fred.stlouisfed.org/series/DFF}},
  note = {Retrieved from FRED, Federal Reserve Bank of St. Louis}
}

@book{costa2005discrete,
  author    = {O. L. V. Costa and M. D. Fragoso and R. P. Marques},
  title     = {Discrete-Time {M}arkov Jump Linear Systems},
  publisher = {Springer-Verlag},
  year      = {2005},
  series    = {Probability and Its Applications},
}

@misc{fred_sp500,
  author = {{S\&P Dow Jones Indices LLC}},
  title = {S\&P 500 [SP500]},
  year = {2026},
  howpublished = {\url{https://fred.stlouisfed.org/series/SP500}},
  note = {Retrieved from FRED, Federal Reserve Bank of St. Louis}
}

@misc{fomc_historical,
  author = {{Board of Governors of the Federal Reserve System}},
  title = {Federal Open Market Committee Historical Materials by Year},
  year = {2026},
  howpublished = {\url{https://www.federalreserve.gov/monetarypolicy/fomc_historical_year.htm}}
}

@article{kuttner2001,
  title={Monetary Policy Surprises and Interest Rates: Evidence from the Fed Funds Futures Market},
  author={Kuttner, Kenneth N.},
  journal={Journal of Monetary Economics},
  volume={47},
  number={3},
  pages={523--544},
  year={2001},
  doi={10.1016/S0304-3932(01)00055-1}
}

@article{bernanke2005,
  title={What Explains the Stock Market's Reaction to Federal Reserve Policy?},
  author={Bernanke, Ben S. and Kuttner, Kenneth N.},
  journal={The Journal of Finance},
  volume={60},
  number={3},
  pages={1221--1257},
  year={2005},
  doi={10.1111/j.1540-6261.2005.00760.x}
}

\appendix

\section{Full Proofs}
\label{app:proofs}

This appendix provides complete proofs of the results stated in Section~\ref{sec:theory}. We first discuss preliminaries and give the full proof of Theorem~\ref{thm:main} (Appendix~\ref{app:thm1}) and Corollary~\ref{cor:known} (Appendix~\ref{app:corollary}).

\subsection{Proof of Theorem~\ref{thm:main}}
\label{app:thm1}

\subsubsection{Notation and Assumptions}

We restate the model first. We observe $N$ input-label pairs generated by the piecewise-linear model
\begin{equation}
y_i = \begin{cases} \langle w_1, x_i \rangle + \varepsilon_i, & i \leq n_1^*, \\ \langle w_2, x_i \rangle + \varepsilon_i, & i > n_1^*, \end{cases}
\tag{\ref{eq:model}}
\end{equation}
with inputs $x_i \in \mathbb{R}^d$ satisfying $\|x_i\| \leq B_x$, labels satisfying $|y_i| \leq B_y$, noise $\varepsilon_i \sim \mathcal{N}(0, \sigma^2)$, and prior $w_1, w_2 \stackrel{\mathrm{iid}}{\sim} \mathcal{N}(0, \lambda^{-1} I_d)$. The change point $n_1^* \in \mathcal{K}_t \subseteq \{1, \dots, t{-}1\}$ has uniform prior $\pi_k = 1/|\mathcal{K}_t|$ and candidate set size $|\mathcal{K}_t| \leq R$.

\paragraph{Sufficient statistics.}
For each candidate change point $k \in \mathcal{K}_t$, define the post-change sufficient statistics
\begin{equation}
A_k(t) = \sum_{i=k+1}^{t-1} x_i x_i^\top, \quad
b_k(t) = \sum_{i=k+1}^{t-1} x_i y_i, \quad
c_k(t) = \sum_{i=k+1}^{t-1} y_i^2.
\label{eq:suffstats}
\end{equation}
We also define the cumulative prefix sums
\begin{equation}
P_j^{xx} = \sum_{i=1}^{j} x_i x_i^\top, \quad
P_j^{xy} = \sum_{i=1}^{j} x_i y_i, \quad
P_j^{yy} = \sum_{i=1}^{j} y_i^2,
\label{eq:prefixsums}
\end{equation}
so that the segment statistics can be recovered by subtraction:
\[
A_k(t) = P_{t-1}^{xx} - P_k^{xx}, \quad
b_k(t) = P_{t-1}^{xy} - P_k^{xy}, \quad
c_k(t) = P_{t-1}^{yy} - P_k^{yy}.
\]

\paragraph{Posterior under a fixed hypothesis.}
For a fixed candidate $k$, the posterior over the post-change parameter $w$ is Gaussian:
\[
w \mid Z_{<t}, k \sim \mathcal{N}(\mu_k, \Sigma_k),
\]
where
\begin{equation}
\Sigma_k = (\lambda I + \sigma^{-2} A_k)^{-1}, \quad
\mu_k = \sigma^{-2} \Sigma_k \, b_k.
\label{eq:posterior}
\end{equation}
The corresponding posterior predictive mean is
\begin{equation}
m_k(t) = x_t^\top \mu_k,
\label{eq:pred_mean_restate}
\end{equation}
aka the ridge regression predictor computed on the post-change segment.

\paragraph{Bayesian model averaging (BMA).}
The BMA predictor maintains a posterior over candidate change points and forms a weighted prediction:
\begin{equation}
\hat{y}_t^{\mathrm{BMA}} = \sum_{k \in \mathcal{K}_t} \alpha_k(t) \, m_k(t),
\label{eq:bma_restate}
\end{equation}
where the posterior weights satisfy
\[
\alpha_k(t) \propto \pi_k \, p(Z_{<t} \mid k), \quad \sum_k \alpha_k(t) = 1,
\]
and $p(Z_{<t} \mid k)$ denotes the marginal likelihood under hypothesis $k$.

\paragraph{Bounded domain.}
All quantities $(A_k, b_k, x_t)$ lie in a bounded set determined by $B_x, B_y,$ and $N$, which ensures that the maps defined below are continuous on a compact domain.

\subsubsection{Restatement of Theorem \ref{thm:main}}

We restate the main approximation result in the notation introduced above.

\begin{theorem}[Transformer approximation of change-point BMA]
\label{thm:main_restate}
Consider the piecewise-linear model~\eqref{eq:model} with bounded inputs and the Gaussian prior and noise model specified above. Let $\hat{y}_t^{\mathrm{BMA}}$ denote the Bayesian model-averaged predictor defined in~\eqref{eq:bma_restate}. 

For any $\epsilon > 0$, there exists a causal transformer $T_\epsilon$ such that, for all $t = 1, \dots, N$ and all valid input sequences,
\begin{equation}
\bigl| T_\epsilon(Z_{<t}, x_t) - \hat{y}_t^{\mathrm{BMA}} \bigr| \leq \epsilon.
\label{eq:approx_restate}
\end{equation}

Moreover, the construction has the following properties:
\begin{itemize}
\item The number of layers is $O(1)$ for fixed $\epsilon$ and problem parameters.
\item The number of attention heads scales as $O(|\mathcal{K}_t|)$.
\item The hidden dimension scales as $O(|\mathcal{K}_t| \cdot d^2)$.
\end{itemize}
\end{theorem}

The proof proceeds by constructing a transformer that (i) computes prefix sufficient statistics, (ii) isolates segment statistics for each candidate $k$, and (iii) approximates the BMA predictor \eqref{eq:bma_restate} using MLP layers.

\subsubsection{Token Embedding}

The initial representation of token $j$ is
\begin{equation}
h_j^{(0)} = \bigl[ x_j;\; y_j;\; \mathrm{vec}(x_j x_j^\top);\; x_j y_j;\; y_j^2;\; \mathrm{PE}(j);\; \mathbf{0} \bigr] \in \mathbb{R}^{D},
\label{eq:embedding}
\end{equation}
where $\mathrm{PE}(j)$ is a positional encoding vector and $\mathbf{0}$ denotes unused coordinates reserved for intermediate computations in subsequent layers. The hidden dimension satisfies $D \geq d + 1 + d^2 + d + 1 + |\mathrm{PE}| + D_{\mathrm{work}}$, where $D_{\mathrm{work}}$ is the working space for intermediate computations.

We prove the theorem by explicit construction of a four-layer causal transformer that approximates the BMA predictor~\eqref{eq:bma} to within $\epsilon$ on all valid input sequences. The construction proceeds layer by layer. We track approximation errors introduced at each stage and show that, by choosing MLP widths sufficiently large, the total error can be made smaller than any prescribed $\epsilon > 0$.

\subsubsection*{Layer 1: Computing Cumulative Prefix Sums}

\paragraph{Step 1a: Uniform causal averaging.}
We configure a single attention head with $W_Q = 0$ and $W_K = 0$. Since all query-key inner products equal zero, the softmax attention weights reduce to the uniform distribution over the causal prefix:
\begin{equation}
\alpha_{ji} = \frac{1}{j}, \quad i = 1, \dots, j.
\label{eq:uniform_attn}
\end{equation}
Setting $W_V$ to be a projection onto the coordinates of $h_i^{(0)}$ containing $S_i^{\mathrm{raw}} := [\mathrm{vec}(x_i x_i^\top);\, x_i y_i;\, y_i^2]$ (which are included in the token embedding), the attention output at position $j$ is the running average
\begin{equation}
\bar{S}_j = \frac{1}{j} \sum_{i=1}^{j} S_i^{\mathrm{raw}}.
\label{eq:running_avg}
\end{equation}
After the residual connection, the representation at position $j$ contains both $\bar{S}_j$ and $\mathrm{PE}(j)$.

\paragraph{Step 1b: Recovering prefix sums via the MLP.}
The cumulative prefix sum at position $j$ is $P_j = j \cdot \bar{S}_j$. Since $\bar{S}_j$ and the positional encoding $\mathrm{PE}(j)$ are both available in the residual stream, the MLP must approximate the map
\begin{equation}
(\bar{S}_j, \mathrm{PE}(j)) \mapsto j \cdot \bar{S}_j = P_j.
\label{eq:mul_map}
\end{equation}
When $\mathrm{PE}(j) = j$ (linear positional encoding), this is multiplication of two bounded scalars (applied entrywise), which is a continuous function on the compact domain $\{(\bar{s}, p) : |\bar{s}| \leq B_x^2 + B_x B_y + B_y^2,\; p \in \{1, \dots, N\}\}$. When $\mathrm{PE}$ is sinusoidal, the map $j \mapsto \mathrm{PE}(j)$ is injective on $\{1, \dots, N\}$. Thus there exists a function $g$ defined on its image such that $g(\mathrm{PE}(j)) = j$. The map $(\bar{S}_j, \mathrm{PE}(j)) \mapsto g(\mathrm{PE}(j)) \cdot \bar{S}_j$ is continuous on a bounded domain, and hence can be uniformly approximated by an MLP.

By the universal approximation theorem, for any $\epsilon_1 > 0$, there exists an MLP such that
\[
\sup_{j = 1, \dots, N} \|\hat{P}_j - P_j\|_\infty \le \epsilon_1.
\]

The result is written to reserved coordinates in $h_j^{(1)}$ via the residual connection.

After Layer 1, the representation at each position $j$ contains $\hat{P}_j \approx P_j$ (the approximate prefix sums), the raw input $(x_j, y_j)$, the positional encoding $\mathrm{PE}(j)$, and the quadratic features $S_j^{\mathrm{raw}}$.

\subsubsection*{Layer 2: Segment Isolation via Retrieval}

At time step $t$, the transformer must retrieve the prefix sums $\hat{P}_k$ at each candidate change-point index $k \in \mathcal{K}_t$ and form segment statistics by subtraction:
\begin{equation}
\hat{A}_k(t) = \hat{P}_{t-1}^{xx} - \hat{P}_k^{xx}, \quad
\hat{b}_k(t) = \hat{P}_{t-1}^{xy} - \hat{P}_k^{xy}, \quad
\hat{c}_k(t) = \hat{P}_{t-1}^{yy} - \hat{P}_k^{yy}.
\label{eq:segment_retrieval}
\end{equation}

\paragraph{Step 2a: Multi-head retrieval of prefix sums.}
For each candidate $k \in \mathcal{K}_t$, we assign a dedicated attention head $h$ with index $h = h(k)$. The number of heads required is $|\mathcal{K}_t| + 1$ (one per candidate plus one head to retrieve $\hat{P}_{t-1}$).

We configure head $h(k)$ so that, at the prediction position $t$, it attends strongly to position $k$. The query-key projections are designed over the positional encoding coordinates:
\begin{equation}
W_Q^{(h)} h_t^{(1)} = \phi(\mathrm{PE}_{\mathrm{target}}(k)), \quad
W_K^{(h)} h_j^{(1)} = \phi(\mathrm{PE}(j)),
\label{eq:retrieval_qk}
\end{equation}
where $\phi$ is chosen so that the vectors $\{\phi(\mathrm{PE}(j))\}_{j=1}^N$ are distinct and can be separated by inner products, i.e., for each $k$,
\[
\langle \phi(\mathrm{PE}(k)), \phi(\mathrm{PE}(k)) \rangle > \langle \phi(\mathrm{PE}(k)), \phi(\mathrm{PE}(j)) \rangle \quad \text{for all } j \ne k.
\] The mechanism depends on how the candidate set $\mathcal{K}_t$ is communicated:

\begin{itemize}
\item \textbf{Support known} ($\mathcal{K}_t = \{L, \dots, U\}$): The PE encodes the boundaries $L, U$. We configure $R = U - L + 1$ heads, with head $h$ having a query projection that targets position $L + h - 1$. Because positional encodings are distinct across positions, the query-key inner product can be made to peak at the target position. More concretely, for linear PE where $\mathrm{PE}(j) = j$, we set $W_Q^{(h)} = [0; \dots; 1; \dots; 0]$ selecting the coordinate encoding target position $k_h$, and $W_K^{(h)}$ selecting the PE coordinate of $h_j^{(1)}$. Scaling these projections by a large constant $C$ makes the attention weights $\alpha_{t,j}^{(h)}$ approximate $\mathbf{1}[j = k_h]$ as $C \to \infty$, so the head retrieves $\hat{P}_{k_h}$ up to softmax leakage that is exponentially small in $C$.

\item \textbf{Known exactly} ($\mathcal{K}_t = \{n_1^*\}$): A single head suffices, with the query keyed to the PE encoding of $n_1^*$.

\item \textbf{No information} ($\mathcal{K}_t = \{1, \dots, t{-}1\}$): In principle, $O(t)$ heads are required. In practice, the training distribution restricts the effective support (e.g., $n_1^* \in \{N/2 - R, \dots, N/2 + R\}$), so $O(R)$ heads suffice for the model to implicitly learn the support.
\end{itemize}

The value projection $W_V^{(h)}$ selects the coordinates of $h_j^{(1)}$ containing $\hat{P}_j$. After the multi-head attention, the residual stream at position $t$ contains approximate copies of $\hat{P}_{k}$ for each $k \in \mathcal{K}_t$, as well as $\hat{P}_{t-1}$ (from either a dedicated retrieval head or the self-retrieval at position $t-1$ via the residual).

\paragraph{Retrieval accuracy.}
Let $\alpha_{t,j}^{(h)}$ denote the attention weights of head $h$ at position $t$. With query-key scaling $C$, the weight on the target position $k_h$ satisfies
\begin{equation}
\alpha_{t,k_h}^{(h)} = \frac{\exp(C)}{\exp(C) + \sum_{j \neq k_h} \exp(s_j)} \geq 1 - (t-1) e^{-C + S_{\max}},
\label{eq:retrieval_acc}
\end{equation}
where $s_j$ are the non-target scores and $S_{\max} = \max_{j \neq k_h} s_j$. Choosing $C$ large enough (dependent on $N$ and $\epsilon$ but not on the data), the retrieved value satisfies $\|\mathrm{Attn}^{(h)}(t) - \hat{P}_{k_h}\| \leq \epsilon_{\mathrm{ret}}$ for any prescribed $\epsilon_{\mathrm{ret}} > 0$.

\paragraph{Step 2b: Forming segment statistics.}
The Layer 2 MLP receives $\hat{P}_{t-1}$ and the retrieved values $\hat{P}_{k}$ for $k \in \mathcal{K}_t$, and computes segment statistics by subtraction~\eqref{eq:segment_retrieval}. Subtraction is a linear operation implemented exactly by the MLP (or even by a linear layer). The approximation error in the segment statistics inherits from the prefix sum errors and retrieval error:
\begin{equation}
\|\hat{A}_k(t) - A_k(t)\|_F \leq 2\epsilon_1 + \epsilon_{\mathrm{ret}},
\label{eq:seg_error}
\end{equation}
and similarly for $\hat{b}_k$ and $\hat{c}_k$. Define $\epsilon_2 := 2\epsilon_1 + \epsilon_{\mathrm{ret}}$.

After Layer~2, the residual stream at the prediction position $t$ contains $(\hat{A}_k, \hat{b}_k, \hat{c}_k)$ for every $k \in \mathcal{K}_t$, the query input $x_t$, and the positional encoding. All subsequent computation occurs only at position $t$, involving no further attention across positions. This ensures causality without backward information flow.

\subsubsection*{Layers 3--4: Posterior Computation and Prediction}

\paragraph{Step 3: Predictive means.}

Define the domain
\[
\mathcal{S} := \left\{ (A, b, x) : A \succeq 0,\ \|A\|_F \le N B_x^2,\ \|b\| \le N B_x B_y,\ \|x\| \le B_x \right\}.
\]

For each candidate $k \in \mathcal{K}_t$, the posterior predictive mean under the Gaussian model is Ridge regression, i.e.,
\begin{equation}
m_k(t) = x_t^\top \mu_k, \quad \text{where} \quad
\mu_k = \sigma^{-2} (\lambda I + \sigma^{-2} A_k)^{-1} b_k.
\label{eq:pred_mean}
\end{equation}
The map $\psi_m : (A_k, b_k, x_t) \mapsto m_k(t)$ is continuous on the compact domain $\mathcal{S}$. In fact, $\psi_m$ is Lipschitz on $\mathcal{S}$. To see this, note that $(\lambda I + \sigma^{-2} A)^{-1}$ is Lipschitz in $A$ on the set $\{A \succeq 0 : \|A\|_F \leq N B_x^2\}$ because
\begin{equation}
\|M_1^{-1} - M_2^{-1}\| \leq \|M_1^{-1}\| \cdot \|M_2^{-1}\| \cdot \|M_1 - M_2\| \leq \lambda^{-2} \|M_1 - M_2\|
\label{eq:inv_lip}
\end{equation}
for $M_i = \lambda I + \sigma^{-2} A_i \succeq \lambda I$. The remaining operations (matrix-vector products, inner products) are bilinear and hence Lipschitz on bounded domains. Composing, there exists a constant $L_m$ depending on $B_x, B_y, \lambda, \sigma, N, d$ such that
\begin{equation}
|m_k(t; \hat{A}_k, \hat{b}_k) - m_k(t; A_k, b_k)| \leq L_m \cdot \epsilon_2.
\label{eq:mk_error}
\end{equation}

By the universal approximation theorem, an MLP with two hidden layers can approximate $\psi_m$ on $\mathcal{S}$ to within any $\epsilon_3 > 0$, using width that depends on $\epsilon_3$, $d$, and the constants $B_x, B_y, \lambda, \sigma$. The MLP computes all $|\mathcal{K}_t|$ maps in parallel by using $|\mathcal{K}_t|$ independent subnetworks within its width. The total error in the predictive mean for candidate $k$ is at most $L_m \epsilon_2 + \epsilon_3$.

\paragraph{Step 4a: Log-marginal likelihoods.}
The log-marginal likelihood under hypothesis $k$ is:
\begin{equation}
\ell_k := \log p(Z_{<t} \mid k) = -\tfrac{n_k}{2} \log(2\pi\sigma^2)
- \tfrac{1}{2} \log\det(\lambda^{-1} A_k + I)
- \tfrac{1}{2\sigma^2}\bigl(c_k - \sigma^{-2} b_k^\top \Sigma_k b_k\bigr),
\label{eq:logml_full}
\end{equation}
where $\Sigma_k = (\lambda I + \sigma^{-2} A_k)^{-1}$ and $n_k = t - 1 - k$ is the number of post-change samples under hypothesis $k$. The first term depends only on $n_k$ (which is determined by $k$ and $t$, both available through the positional encoding and the position in the sequence). The second and third terms are continuous functions of $(A_k, b_k, c_k)$ on $\mathcal{S}$ by Lemma~\ref{lem:continuity}(ii).

The map $\psi_\ell : (A_k, b_k, c_k, n_k) \mapsto \ell_k$ is again Lipschitz on $\mathcal{S}$, with Lipschitz constant $L_\ell$ depending on the problem parameters. In particular:
\begin{itemize}
\item The log-determinant satisfies $|\log\det(M_1) - \log\det(M_2)| \leq d \cdot \|M_1^{-1}\| \cdot \|M_1 - M_2\|$ for nearby positive definite matrices $M_1, M_2 \succeq I$, so the log-determinant term is Lipschitz in $A_k$ with constant depending on $d$.
\item The quadratic form $b_k^\top \Sigma_k b_k$ is Lipschitz in $(A_k, b_k)$ by the same argument as~\eqref{eq:inv_lip} combined with bilinearity.
\item The term $c_k$ enters linearly.
\end{itemize}
Therefore $|\ell_k(\hat{A}_k, \hat{b}_k, \hat{c}_k) - \ell_k(A_k, b_k, c_k)| \leq L_\ell \cdot \epsilon_2$.

An MLP approximates $\psi_\ell$ on $\mathcal{S}$ to within $\epsilon_4 > 0$ for each of the $|\mathcal{K}_t|$ candidates. Total error in each log-marginal likelihood: $L_\ell \epsilon_2 + \epsilon_4$.

\paragraph{Step 4b: Posterior weights and weighted prediction.}
The posterior weights are $\alpha_k(t) = \mathrm{softmax}_k(\ell_1, \dots, \ell_{|\mathcal{K}_t|})$. The softmax function is Lipschitz on bounded domains: for $\ell, \ell' \in [-M, M]^R$,
\begin{equation}
\|\mathrm{softmax}(\ell) - \mathrm{softmax}(\ell')\|_1 \leq 2\|\ell - \ell'\|_\infty.
\label{eq:softmax_lip}
\end{equation}
(This follows from the fact that each partial derivative of softmax is bounded by~1.) Let $\hat{\ell}_k$ denote the MLP's approximation to $\ell_k$. The error in $\hat{\ell}_k$ is at most $L_\ell \epsilon_2 + \epsilon_4$, so the posterior weight error is bounded by
\begin{equation}
\|\hat{\alpha} - \alpha\|_1 \leq 2(L_\ell \epsilon_2 + \epsilon_4).
\label{eq:alpha_error}
\end{equation}

Finally, the BMA prediction is the weighted sum $\hat{y}_t^{\mathrm{BMA}} = \sum_k \alpha_k m_k$. The MLP computes $\hat{y}_t = \sum_k \hat{\alpha}_k \hat{m}_k$. The total error decomposes as
\begin{align}
\bigl|\hat{y}_t - \hat{y}_t^{\mathrm{BMA}}\bigr|
&= \biggl|\sum_k \hat{\alpha}_k \hat{m}_k - \sum_k \alpha_k m_k\biggr| \nonumber \\
&\leq \biggl|\sum_k (\hat{\alpha}_k - \alpha_k) \hat{m}_k\biggr|
+ \biggl|\sum_k \alpha_k (\hat{m}_k - m_k)\biggr| \nonumber \\
&\leq \|\hat{\alpha} - \alpha\|_1 \cdot \max_k |\hat{m}_k|
+ \max_k |\hat{m}_k - m_k|,
\label{eq:total_error}
\end{align}
where we used $\sum_k |\alpha_k| = 1$ and $\sum_k |\hat{\alpha}_k - \alpha_k| = \|\hat{\alpha} - \alpha\|_1$. The predictive means are bounded: $|m_k| \leq B_x \cdot \|\mu_k\| \leq B_x \lambda^{-1} \sigma^{-2} N B_x B_y =: M_m$, and the approximate means satisfy $|\hat{m}_k| \leq M_m + L_m \epsilon_2 + \epsilon_3$.

Combining all error terms:
\begin{equation}
\bigl|\hat{y}_t - \hat{y}_t^{\mathrm{BMA}}\bigr|
\leq 2(L_\ell \epsilon_2 + \epsilon_4)(M_m + L_m \epsilon_2 + \epsilon_3) + (L_m \epsilon_2 + \epsilon_3).
\label{eq:total_error_bound}
\end{equation}

\paragraph{Choosing tolerances.}
Given $\epsilon > 0$, we choose:
\begin{enumerate}
\item $\epsilon_1$ (prefix sum MLP error) small enough,
\item $\epsilon_{\mathrm{ret}}$ (retrieval accuracy, controlled by the query-key scaling $C$) small enough,
\item $\epsilon_3$ (predictive mean MLP error) small enough,
\item $\epsilon_4$ (log-marginal likelihood MLP error) small enough,
\end{enumerate}
such that $\epsilon_2 = 2\epsilon_1 + \epsilon_{\mathrm{ret}}$ and the total error~\eqref{eq:total_error_bound} is at most $\epsilon$. Since~\eqref{eq:total_error_bound} is a polynomial in $(\epsilon_1, \epsilon_{\mathrm{ret}}, \epsilon_3, \epsilon_4)$ with no constant term, this can always be achieved by choosing each tolerance sufficiently small (with the required smallness depending on the problem parameters $B_x, B_y, \lambda, \sigma, N, d$).

\paragraph{Resource accounting.}
The construction uses:
\begin{itemize}
\item \textbf{Layers:} 4 (one for prefix sums, one for retrieval and subtraction, two for posterior computation). The number of layers is $O(1)$ for fixed $\epsilon$ and problem parameters.
\item \textbf{Attention heads:} Layer~1 uses 1 head (uniform averaging). Layer~2 uses $|\mathcal{K}_t| + 1 \leq R + 1$ heads. Layers~3--4 require no attention (MLP-only computation at the prediction token). Total: $O(R)$ heads.
\item \textbf{Hidden dimension:} Must accommodate $d^2 + d + 1$ coordinates for each of the $|\mathcal{K}_t|$ candidates' sufficient statistics, plus $d$ coordinates for $x_t$, plus working space for the MLP computations. Total: $O(R \cdot d^2 + D_{\mathrm{MLP}})$, where $D_{\mathrm{MLP}}$ depends on $\epsilon$ through the MLP width needed for universal approximation.
\end{itemize}
This completes the proof. \qed

\subsection{Proof of Corollary~\ref{cor:known}}
\label{app:corollary}

When $n_1^*$ is known to the transformer via positional encoding, the candidate set reduces to $\mathcal{K}_t = \{n_1^*\}$ for $t > n_1^*$. The BMA predictor~\eqref{eq:bma} collapses to a single term:
\[
\hat{y}_t^{\mathrm{BMA}} = m_{n_1^*}(t) = x_t^\top \mu_{n_1^*},
\]
since $\alpha_{n_1^*}(t) = 1$. The construction from Theorem~\ref{thm:main} simplifies as follows.

\paragraph{Layer 1 (accumulation).} Instead of computing prefix sums at every position and retrieving later, the transformer can directly accumulate only post-change statistics. A single attention head uses positional masking: the query-key projection is configured so that position $t$ attends only to positions $j > n_1^*$, which is achievable because $n_1^*$ is encoded in the positional features. Concretely, with linear PE and the change-point index appended to each token's representation, the key projection at position $j$ can encode $\mathbf{1}[j > n_1^*]$ as a large positive or negative bias, effectively masking pre-change tokens. This produces:
\[
\mathrm{Attn}(t) \approx \frac{1}{t - n_1^* - 1} \sum_{i=n_1^*+1}^{t-1} S_i^{\mathrm{raw}},
\]
which, after multiplication by $(t - n_1^* - 1)$ via the MLP (using the PE), yields the segment statistics $A_{n_1^*}(t), b_{n_1^*}(t), c_{n_1^*}(t)$ directly.

\paragraph{Layer 2 (posterior mean).} An MLP approximates
\[
(A_{n_1^*}, b_{n_1^*}, x_t) \mapsto x_t^\top (\lambda I + \sigma^{-2} A_{n_1^*})^{-1} \sigma^{-2} b_{n_1^*}
\]
by the same argument as in Theorem~\ref{thm:main}, Step~3.

\paragraph{No posterior weighting needed.} Since $|\mathcal{K}_t| = 1$, there is no need to compute log-marginal likelihoods, perform softmax normalization, or form weighted averages. The MLP layers devoted to these computations in the full construction are eliminated entirely.

\paragraph{Resources.} The known-change-point construction requires $O(1)$ layers, $O(1)$ attention heads, and hidden dimension $O(d^2)$ (to store a single set of sufficient statistics), with MLP width depending on $\epsilon$ and the problem parameters.

\end{document}